\documentclass{article}
\usepackage{graphicx}
\usepackage{physics}
\usepackage{amsmath}
\usepackage{amssymb}
\usepackage{amsthm}
\usepackage{mathtools}
\usepackage{a4wide}
\usepackage{authblk}
\usepackage{xurl}
\usepackage{tikz} 
\usepackage{pgfplots}
\usepackage{hyperref}
\pgfplotsset{compat=1.18}
\usepackage{multirow}
\newtheorem{theorem}{Theorem}[section]
\newtheorem{definition}[theorem]{Definition}

\usepackage[letterpaper,top=2cm,bottom=2cm,left=3cm,right=3cm,marginparwidth=1.75cm]{geometry}

\usepackage{siunitx}
\title{Noise-based reward-modulated learning}

\date{}

\author[]{Jesús García Fernández}
\author[]{Nasir Ahmad}
\author[]{Marcel van Gerven}

\affil[1]{Department of Machine Learning and Neural Computing, Donders Institute for Brain, Cognition and Behaviour\\
  Radboud University, Nijmegen, the Netherlands}

\begin{document}
\maketitle

\begin{abstract}

The pursuit of energy-efficient and adaptive artificial intelligence (AI) has positioned neuromorphic computing as a promising alternative to conventional computing. However, achieving learning on these platforms requires techniques that prioritize local information while enabling effective credit assignment. Here, we propose noise-based reward-modulated learning (NRL), a novel synaptic plasticity rule that mathematically unifies reinforcement learning and gradient-based optimization with biologically-inspired local updates. NRL addresses the computational bottleneck of exact gradients by approximating them through stochastic neural activity, transforming the inherent noise of biological and neuromorphic substrates into a functional resource. Drawing inspiration from biological learning, our method uses reward prediction errors as its optimization target to generate increasingly advantageous behavior, and eligibility traces to facilitate retrospective credit assignment. Experimental validation on reinforcement tasks, featuring immediate and delayed rewards, shows that NRL achieves performance comparable to baselines optimized using backpropagation, although with slower convergence, while showing significantly superior performance and scalability in multi-layer networks compared to reward-modulated Hebbian learning (RMHL), the most prominent similar approach. While tested on simple architectures, the results highlight the potential of noise-driven, brain-inspired learning for low-power adaptive systems, particularly in computing substrates with locality constraints. NRL offers a theoretically grounded paradigm well-suited for the event-driven characteristics of next-generation neuromorphic AI.

\end{abstract}


\section{Introduction}

Modern artificial intelligence (AI) models are growing exponentially, leading to proportional increases in their computational demands and energy consumption~\cite{patterson2021carbon}. This trend, combined with the growing requirement for widespread AI deployment in local edge devices~\cite{singh2023edge, wang2025optimizing}, creates a massive demand for systems that are both fast and energy-efficient. Meeting this demand requires a fundamental shift in computing architecture to overcome the energy and latency bottlenecks of conventional Von Neumann systems~\cite{von1993first}. Neuromorphic computing~\cite{schuman2022opportunities, roy2019towards, kudithipudi2025neuromorphic} offers a promising path forward by mimicking the low energy, event-driven, and parallel processing capabilities of the brain~\cite{schuman2022opportunities, maass2014noise}. Although neuromorphic hardware is highly efficient for inference, the ability to perform effective on-chip learning and adaptation remains the primary challenge, lacking suitable on-chip learning algorithms.

Backpropagation (BP) remains the workhorse algorithm for training artificial neural networks, computing precise global gradients for weight updates~\cite{rumelhart1986learning}. However, its operational requirements are fundamentally incompatible with the constraints of neuromorphic hardware~\cite{esser2015backpropagation}. BP requires precise global error propagation, which translates into high-bandwidth global data movement and large memory requirements to store activations, the main source of energy consumption in deep learning systems~\cite{marblestone2016toward}. This contrasts sharply with the hardware needs for local synaptic updates that rely only on information available at the synaptic weight. In addition, BP's iterative forward and backward passes are incompatible with the asynchronous event-driven regime in which most neuromorphic devices operate~\cite{neftci2017event}, and its deterministic nature cannot tolerate inherent noise in neuromorphic circuits. 

In response, local plasticity approaches and biologically-inspired learning rules have been proposed to bypass the constraints of BP~\cite{whittington2019theories, whittington2017approximation, lee2015difference, schmidgall2024brain, illing2019biologically, yi2023learning}. A particularly promising family of methods is noise-based learning, which leverages stochastic perturbations (either in weights~\cite{Fernandez2024Ornstein,zuge2023weight,cauwenberghs1992fast, dembo1990model} or activations~\cite{flower1992summed, hiratani2022stability, williams1992simple, werfel2003learning, fiete2006gradient}) to approximate the true direction of the gradient. By injecting small random perturbations, these methods avoid the need for perfectly matched feedback pathways and can leverage inherent noise~\cite{seung2003learning, maass2014noise}. 

Given the prevalence of noise on neuromorphic and biological substrates~\cite{faisal2008noise, stein2005neuronal, mcdonnell2011benefits}, using noise as a mechanism to learn synaptic weights is an area of growing interest, with reward-modulated Hebbian learning (RMHL)~\cite{legenstein2010reward, miconi2017biologically} as a promising candidate. RMHL, which is a form of three-factor Hebbian learning~\cite{fremaux2016neuromodulated,kusmierz2017learning}, offers a potentially powerful mechanism for credit assignment without explicit backpropagation of errors. Despite theoretical advances, few noise-based methods have been adapted to real-world tasks. Some attempts which have been made include application to control problems~\cite{burms2015reward}, for spatiotemporal pattern generation~\cite{kawai2023spatiotemporal}, and for training non-differentiable spiking neural networks~\cite{ferrari2008biologically}, and these particularly struggle with delayed reward regimes. This gap has limited the integration of noise-based approaches for effective credit assignment in temporally extended setups.

To bridge these gaps, we propose noise-based reward-modulated learning (NRL), a biologically-inspired, gradient-free learning method, which is compatible with delayed feedback/rewards. We derive our learning rule from first principles, employing directional derivatives to compute a local gradient estimate at each synapse. These directional derivatives are implemented through stochastic neurons, aligning with the noisy nature of neuromorphic physical systems, and are estimated with two forward passes: a “noisy” pass with stochastic neurons and a “clean” pass without noise. In scenarios where a noiseless pass is infeasible, multiple noisy passes can be averaged to approximate the clean pass, maintaining performance as shown in our experiments. NRL draws from established neuroscientific concepts, by employing eligibility traces~\cite{gerstner2018eligibility, izhikevich2007solving}, which tag synapses based on recent pre/post activity and noise perturbations, allowing retrospective credit assignment over behavioral timescales. It also incorporates reward prediction errors (RPEs), a dopamine-like signal~\cite{schultz1997neural, schultz2016dopamine, montague1996framework} that modulates eligibility traces to reinforce actions that yield unexpected positive rewards, driving learning toward increasingly more advantageous behaviors. We deliberately utilize a rate-based neural model, instead of a spiking model, to focus the analysis on the core computational capabilities and limitations of the learning rule itself, making our findings immediately applicable to neuromorphic substrates.

We validate NRL on a suite of reinforcement-learning benchmarks that span immediate-reward and delayed-reward tasks, demonstrating significantly superior learning efficiency compared to RMHL (particularly when rewards were delayed by many steps) and competitive final performance compared to the gradient-based baseline. Moreover, unlike RMHL, which breaks down in deeper architectures, NRL scales to multilayer networks, though converging more slowly due to its intrinsic stochasticity. 

Our findings argue that NRL represents a promising and scalable step toward developing low-energy and fast adaptive AI systems. By transforming the stochasticity of neuromorphic hardware from a challenge into a computational resource, NRL provides a crucial new algorithmic path for enabling complex, continuous learning at the edge, as well as novel, alternative, robust learning schemes for machine learning.

\section{Methods}

\subsection{Neural system-environment interaction and reward signaling}
\label{agent_setup}

We model a neural system within a dynamic environment, grounded in the reinforcement learning framework~\cite{sutton2018reinforcement}. In this framework, the state of the environment at time $t$ is denoted by $s_t$. The neural system perceives $s_t$ and responds by initiating actions $a_t$, which influence the environment. Each action or sequence of actions results in feedback in the form of positive or negative rewards $r_t$. This iterative process of perception, action, and outcome forms a feedback loop that enables the system to learn and adapt to the changing environment, ultimately seeking to maximize the rewards it receives.

The decision-making of the system is guided by a policy $\pi$, which represents the probability of choosing an action, $a_t$, in a given state $s_t$ at time $t$, i.e., $\pi(a_t \mid s_t)$. This policy captures the system's learned strategy for choosing actions that are likely to yield favorable outcomes. To determine an action, the system computes a probability distribution across the set of possible actions, modeled as a categorical distribution $P(a_t \mid  s_t) \equiv \pi(a_t \mid s_t)$ 
with $\sum_{a \in A} \pi(a \mid s_t) = 1$, where $A$ denotes the set of all possible actions.

The rewards $r_t$, obtained from the environment, serve as crucial learning signals. Internally, the system maintains a prediction of the expected reward, denoted by $\bar{r}_t$. Although this prediction can be calculated through various mechanisms, we model it here for simplicity as a running average of recent rewards as
\begin{align}
    \bar r_{t+1} =\bar r_{t} +  \lambda \left( r_{t+1} - \bar r_{t} \right) 
\end{align}
where $\lambda$ is a smoothing factor that governs the influence of past rewards on the current prediction. The mismatch between the actual reward and the predicted reward, known as the reward prediction error, is then computed as
\begin{align}
 \delta_t = r_t - \bar r_t \,.
\end{align}
This RPE signal, which draws a strong parallel to the phasic activity of dopamine neurons in the brain~\cite{schultz1997neural, schultz2016dopamine, montague1996framework, fiorillo2003discrete}, functions as a low-bandwidth global feedback signal that is inherently compatible with the communication constraints of neuromorphic hardware. Although our running average is a simplified representation of how reward expectations are formed, it effectively captures the fundamental dynamics of the RPE and its role as an adaptive learning signal. We will explore potential refinements of this reward prediction mechanism in the Discussion section.

\subsection{Derivation of the learning rule}
\label{derivation_section}

This section presents a complete derivation of the NRL update rule. We begin by establishing a gradient-based learning rule rooted in an optimization target that aims to maximize unexpected rewards from the environment by increasing the reward prediction error. Using RPE, rather than direct reward, enables our system to adapt dynamically to changing environments by tracking how rewards deviate from expectations, improving long-term performance by continuously seeking rewards that exceed expectations, and maintaining robustness against variations in reward structures~\cite{sutton2018reinforcement, o2017learning}. 

Building on this foundation, we transition to a directional derivative framework, where intrinsic noise within the network is leveraged to approximate gradients. This critical step eliminates the need for backpropagation and feedback phases, enabling learning through forward passes alone. Finally, we extend our noise-based approach to handle scenarios with delayed rewards, a hallmark of real-world problems. This extension enables NRL to adapt to environments where feedback is obtained only after a series of actions.

\subsubsection{Gradient-based learning rule}

The system is modeled as a multi-layered network of interconnected units, analogous to populations of neurons, which processes input representing sensory observations and produces a probability distribution over possible actions. For simplicity, we refer to these populations as `layers,' where each layer performs a transformation of its input in two stages: first, a linear transformation, followed by a non-linear transformation. This can be expressed mathematically as
\begin{align}
    x^l_t = f(h^l_t) = f(W^l_t x^{l-1}_t)
\label{eq:layer}
\end{align}
where $x^l_t$ represents the output activity of layer $l$, $h^l_t$ represents the layer pre-activation (the combined input to such layer), $W^l_t$ represents the matrix of synaptic weights connecting layer $l-1$ to layer $l$, at time $t$, and $f(\cdot)$ is a non-linear activation function, which introduces crucial non-linearities into the network's computations 

In this setup, the goal is to adjust the synaptic weights, $\{ W^1, \ldots, W^L \}$, across the $L$ layers of the network to maximize the rewards obtained by the system. Our derivation begins by introducing a general parameter $\theta$, which will later be mapped to the specific neural network parameters $W$. The primary learning objective is to maximize the RPE, $\delta_t$,  at any particular moment in time, $t$, effectively driving the system to seek out actions that lead to higher-than-predicted rewards. The reward prediction is modeled as a running average of recent rewards and, as such, the learning rule seeks to outperform previously received rewards. Note that for a fixed reward prediction, this is equivalent to maximizing the reward itself. We express this objective in terms of a parameterized policy at time $t$ as 
\begin{align}
     J(\theta_t) = \mathbb{E}_{\pi_{\theta_t}} \left[\delta_t \right] \,. 
\end{align}
To optimize $J(\theta_t)$, we can incrementally update the weights using the gradients with respect to $\theta_t$ as
\begin{align}
     \theta_{t+1} \leftarrow \theta_t + \eta \nabla J(\theta_t)
\end{align}
where $\eta$ is the learning rate, controlling the size of each update step.

Applying the policy gradient theorem~\cite{sutton1999policy}, which utilizes the likelihood-ratio method, we express the gradient of the objective as
\begin{equation*}
\nabla J(\theta_t)=\mathbb{E}_{\pi_{\theta_t}}\left[\nabla \log \pi_{\theta_t} \left(a_t \mid s_t\right) \delta_t\right]\,.
\end{equation*}
For empirical estimation, in the case of a single sample, we approximate the gradient as
$
    \nabla J(\theta_t) \approx \nabla \log \pi_{\theta_t} \left(a_t \mid s_t\right) \delta_t
$. 
Using this approximation, we define the parameter updates as
\begin{align}
     \theta_{t+1} \leftarrow \theta_t + \eta \; \nabla \log \pi_{\theta_t} \left(a_t \mid s_t\right) \; \delta_t \,. 
\label{eq:update_grad}
\end{align}

Equation~\eqref{eq:update_grad} resembles the REINFORCE update rule~\cite{williams1992simple} but differs by using the reward prediction error, $\delta_t$, as the learning signal instead of cumulative rewards over full trajectories. This RPE-based approach leverages immediate feedback from rewards as they are obtained rather than requiring a full trial completion to estimate the policy gradient. It shares conceptual similarities with actor-critic methods in reinforcement learning~\cite{barto1983neuronlike}, where the policy is adjusted using a temporal difference error. However, we approximate future rewards with a running average of past rewards instead of a critic network,  a less powerful but simpler implementation which maintains adaptive feedback.

\subsubsection{Noise-based learning rule}

The learning rule derived in the previous section still relies on gradient descent.
To avoid using backpropagation to compute the gradients, we propose a noise-based alternative that extends Equation~\eqref{eq:update_grad}. This approach leverages gradient approximation via directional derivatives, enabling a theoretically rigorous derivation of noise-driven learning.

A directional derivative quantifies the rate of change of a function in a specified direction. In our neural system, we implement this concept by introducing random noise into parameters. By comparing the network’s parameters with and without this noise, we obtain an estimate of the gradient direction.
To formalize this, we define $g(\theta_t) = \log \pi_{\theta_t}(a_t \mid s_t)$ and express the gradient term in terms of directional derivatives using the theorems in Appendix \ref{appendix:directional_derivatives} as
\begin{align}
    \nabla g(\theta_t) &= \nabla \log \pi_{\theta_t} (a_t \mid s_t)
= n \mathbb{E} \left[ \bar{\epsilon}_t \nabla_{\bar{\epsilon}_t} g(\theta_t) \right] 
\end{align}
where $\bar{\epsilon}_t = \flatfrac{\epsilon_t}{||\epsilon_t||}$ is a normalized direction vector derived from noise $\epsilon_t \sim \mathcal{N}(0, \sigma^2 I_{n})$ with $n$ the number of parameters.

We may expand the above to approximate the gradient via a finite-difference and by sampling under an empirical distribution
\begin{align}
    \nabla g(\theta_t) 
    \approx n \sum_{i=1}^K \left[ \frac{\epsilon^{(i)}_t}{||\epsilon^{(i)}_t||^2} \left( g\bigl(\tilde{\theta}^{(i)}_t \bigr)-g\bigl(\theta_t \bigr) \right) \right]
\end{align}
where $K$ denotes the number of samples and $\tilde{\theta}_t^{(i)} = \theta_t + \epsilon^{(i)}_t$ the noise-perturbed parameters (see Appendix \ref{appendix:directional_derivatives}). 

In practice, we consider $K = 1$, analogous to single-sample updates in stochastic gradient descent. Although increasing $K$ (using multiple noise samples per step) would reduce the variance of the gradient estimate, it would linearly increase the computational cost and latency by requiring $K$ forward passes per update. Here, $g\bigl(\tilde{\theta}^{(i)}_t \bigr) = \log \pi_{ \tilde \theta_t} (a_t \mid s_t)$ and  $g\left(\theta_t \right) = \log \pi_{\theta_t} (a_t \mid s_t)$ are the log of the noise-perturbed and noise-free output, respectively. Thus, we define $\rho_t = \log \pi_{ \tilde \theta_t} (a_t \mid s_t) - \log \pi_{ \theta_t} (a_t \mid s_t)$, which captures the impact of the noise on the policy. Computationally, this term measures the influence of the perturbation on the action selection, as it quantifies how much the injected noise increased or decreased the likelihood of the chosen action compared to the noise-free baseline. Putting this together for the $K=1$ case, we obtain our noise-based learning rule
\begin{equation*}
    \theta_{t+1} \gets \theta_t + \eta \delta_t \hat{\epsilon}_t \rho_t
\end{equation*}
with $\hat{\epsilon}_t = \flatfrac{\epsilon_t}{||\epsilon_t||^2}$. For convenience, we absorbed the constant $n$ into the learning rate $\eta$.

\subsubsection{Learning with node level noise in neural network}

The noise-based learning rule derived so far is formulated for a general parameter $\theta_t$. Now, we apply this framework to the specific context of a neural network, where $\theta_t$ corresponds to the synaptic weights $W^l_t$ of layer $l$ at time $t$. 

Instead of directly perturbing the weights of the network $W^l_t$, we propose introducing noise directly into the neurons (nodes) of each layer. This strategy is advantageous because it leads to reduced variance in gradient estimation, as perturbations occur in a lower-dimensional space (the neural activations) compared to the full weight space. In addition, it naturally reduces the computational cost and communication overhead associated with high-dimensional perturbations. This approach shares similarities with node perturbation (NP)~\cite{flower1992summed, hiratani2022stability, williams1992simple, werfel2003learning, fiete2006gradient}, which has shown benefits in reducing gradient variance and offering more localized updates.

Given a layer transformation as in Equation~\eqref{eq:layer}, adding noise at the neuron level is represented as
\begin{align}
    \tilde x^l_t = f\left(\tilde h^l_t + \xi^l_t \right) = f\left(W^l_t (\tilde{x}^{l-1}_t)^\top + \xi^l_t \right)
\end{align}
where $\xi^l_t \sim \mathcal{N}(0, \sigma^2 I_{m^l})$ is the injected noise at time $t$, $\sigma^2$ is some arbitrarily small noise scale, and $m^l$ is the number of neurons in layer $l$. We use the notation $\tilde{x}^l_t$ and $\tilde{h}^l_t$ to denote perturbed inputs and pre-activations, respectively, which may also result from perturbations of previous network layers on which layer $l$ depends.

To formulate the learning rule for specific layer parameters $W^l_t \in \{ W^1_t, \ldots, W^L_t \}$, we re-express the gradient in terms of the pre-activations instead of the parameters themselves. Here, the general parameter vector $\theta_t=\{ W^1_t, \ldots, W^L_t \}$ corresponds to the full set of all weight matrices for all layers at time $t$. To do so, we first rewrite the gradient term as
$
   \nabla g\left(W^l_t \right) = \nabla \log \pi_{\theta_t} \left(a_t \mid s_t\right)
$
and apply the chain rule to break down the gradient with respect to $W^l_t$ as
\begin{equation}
\nabla g\left(W^l_t \right) = \nabla g \left(h^l_t  \right) \frac{\partial h^l_t}{\partial W^l_t} = \nabla  g \left(h^l_t \right) (x^{l-1}_t)^\top.
\end{equation}
Notably, we may once again carry out the conversion from this gradient estimation step to a set of directional derivatives such that
\begin{equation}
\nabla g \left(h^l_t  \right) (x^{l-1}_t)^\top \approx  n \sum_{i=1}^K \left[ \frac{\xi^{l}_{t,i}}{||\xi^{l}_{t,i}||^2} \left( g\bigl(\tilde{h}^{l}_{t,i} \bigr)-g\bigl(h^l_t \bigr) \right) \right] (x^{l-1}_t)^\top
\end{equation}
where $i$ indexes over a set of repeated samples of the noise term $\xi$. 
Again, as above, by reducing this to the single sample ($n=1$) case, we can define a layer-specific weight update rule
\begin{equation}
   W^l_{t+1} \gets W^l_t + \eta \delta_t \bar{\xi}^{l}_t {\rho}_t (\tilde{x}^{l-1}_t)^\top
\label{eq:update_noise}
\end{equation}
with ${\rho_t} = \log \pi_{\tilde W_t}(a_t \mid s_t) - \log \pi_{W_t}(a_t \mid s_t)$ and $\bar{\xi}^{l}_t = \flatfrac{\xi^l_t}{||\xi^l_t||^2}$. Here, $\pi_{\tilde W_t}(a_t \mid s_t)$ represents the network's output when noise is injected into the neurons, while $\pi_{ W_t}(a_t \mid s_t)$ corresponds to the output of the noiseless network. This form of the noise-based learning rule is directly applicable in settings with continuous reward signals.

\subsubsection{Learning from sparse rewards}

Our derivations so far assume that synaptic updates happen at every time step, implying a continuous stream of rewards and learning signals. In reality, however, rewards are often sparse, arriving only after a sequence of actions or upon reaching specific milestones.  To handle this, we will modify our learning rule to account for rewards received at arbitrary times.

Let us denote these discrete reward times as $\tau_m \in \{\tau_0, \tau_1, \dots, \tau_M\}$. At these moments, we compute the RPE
as
\begin{align}
 \delta_{\tau_m} = r_{\tau_m} - \bar r_{\tau_m}.
\end{align}
Here, $r_{\tau_m}$ is the reward received at time $\tau_m$, and $\bar r_{\tau_m}$ is the reward prediction, which is updated as the running average of recent rewards:
\begin{align}
    \bar r_{\tau_m} =\bar r_{\tau_{m-1}} +  \lambda \left( r_{\tau_m} - \bar r_{\tau_{m-1}} \right). 
\end{align}
While synaptic updates only occur at these specific reward times $\tau_m$, the learning rule continuously accumulates information between rewards. This information is integrated over time, allowing us to update weights based on what's happened since the last reward. This leads to our modified learning rule
\begin{equation}
    W^l_{\tau_m} \gets W^l_{\tau_{m-1}} + \eta \; \delta_{\tau_m}\; e_{\tau_m} 
\label{eq:final_update}
\end{equation}
where 
\begin{align}
    e_{\tau_m} = \sum_{t=\tau_{m-1}}^{\tau_m} \bar{\xi}^{l}_t {\rho}_t (\tilde{x}^{l-1}_t)^\top
\label{eq:eligibility_trace}
\end{align}
is an eligibility trace that acts as a mechanism to connect past actions with future rewards~\cite{gerstner2018eligibility}. Eligibility traces capture neural activity and other local variables over time, signaling potential synaptic changes. Upon receiving a reward, these traces are modulated by the reinforcement signal, resulting in synaptic updates. Some models view eligibility traces as decaying cumulative activity~\cite{izhikevich2007solving, sutton2018reinforcement}), while others treat them as a full activity history~\cite{miconi2017biologically}, which aligns with our formulation. The eligibility trace efficiently tracks neural information between rewards, facilitating the assignment of credit to past actions. 

Thus, our final learning rule consists of two primary components: (i) an eligibility trace, defined in Equation~\eqref{eq:eligibility_trace}, which accumulates local information over time at each time step, and (ii) a synaptic update, defined in Equation~\eqref{eq:final_update}, triggered upon reward receipt, which modulates the eligibility trace to adjust the synaptic weights. This learning rule constitutes the core of our proposed noise-based reward-modulated Learning (NRL) and is used for all experiments presented in this paper.

\subsection{Neural network architecture}
\label{nn_setup}

The learning and decision-making capabilities of the system are modelled using a feedforward network with $L$ layers, representing interconnected populations of neurons. The transformation performed by each hidden layer at time $t$ can be expressed as
\begin{align}
      \tilde x^l_t &= f \left( W^l_t (\tilde{x}^{l-1}_t)^\top + \xi^l_t \right) 
\end{align}
and the readout in the output layer is given by
\begin{align}
      \tilde y_t &= s \left( W^l_t (\tilde{x}^{l-1}_t)^\top + \xi^l_t \right)
\end{align}
where $x^l_t$ represents the activity of units in layer $l$, $W^l_t$ is the matrix of synaptic weights connecting layer $l-1$ to layer $l$, and $\xi^l_t \sim \mathcal{N}(0, \sigma^2 I_{m^l})$ represents random noise introduced into layer $l$, with $m^l$ being the number of units in such layer. This noisy propagation of activity is termed a ``noisy pass.'' The function $f(\cdot)$ is a non-linear activation function, which we implement as the LeakyReLU function such that $f(x) = x$ if $x\geq 0$ and $f(x) = \alpha x$ with $0<\alpha \ll 1$. The function $s(\cdot)$ represents a softmax transformation ${s}_i(x) = \flatfrac{e^{x_i}}{\sum_{j=1}^{m^l} e^{x_j}}$, which transforms the network's output into a probability distribution over possible actions. We also define a ``clean pass," representing the network's activity in the absence of noise, by $x^l_t = f \left( W^l_t (x^{l-1}_t)^\top \right)$ and $y_t = s \left( W^l_t (x^{l-1}_t)^\top \right)$.

\subsection{Experimental validation}

We validate our approach across a range of simulated environments, comparing it against established baseline learning methods. Each environment presents a problem with discrete episodes, where the system’s learning is guided solely by positive or negative rewards. We investigate both scenarios with immediate reward and those involving delayed reward. In the delayed reward scenarios, a single reward is provided after a sequence of actions, challenging our system to assign credit retrospectively to the actions that contributed to the outcome. For the instantaneous reward setting, we use the Reaching problem~\cite{georgopoulos1986neuronal}, while the delayed reward setting uses the Cartpole~\cite{barto1983neuronlike} and Acrobot~\cite{sutton1995generalization} problems. We use implementations given by the libraries OpenAI Gym~\cite{brockman2016openai} and NeuroGym~\cite{molano2022neurogym} for the different environments.

We compare NRL to two baselines: an exact-gradient version of NRL, which serves as an ``optimal performance'' benchmark, and an RMHL approach.

The first baseline, which we refer to as the exact-gradient method (BP baseline), is similar to an actor-only variant of actor-critic methods, relying on a running average of past rewards as the prediction error. For this baseline, we compute the exact policy gradients for the entire episode trajectory using backpropagation through time (BPTT)~\cite{Werbos1990}. Given sparse rewards, the weights at each layer are updated at the end of each episode according to 
\begin{align}
W^l_{\tau_m}\gets W^l_{\tau_{m-1}}+\eta \left( \nabla_{W^l_{\tau_{m-1}}}\log\pi_{\theta_{\tau_{m-1}}}(a_t\mid s_t) \right ) \delta_{\tau_m}.
\label{eq:BP_update_sparse}
\end{align}

The second baseline, motivated by the noise-based nature of NRL, is the RMHL rule from~\cite{legenstein2010reward} with an explicit-noise approximation -- the original version where noise is inferred from neural activities proved too unstable for the problems considered here. For delayed rewards, we adapt it similarly to~\cite{miconi2017biologically}, updating eligibility traces as
\begin{align}
    e_{\tau_m} = \sum_{t=\tau_{m-1}}^{\tau_m} \xi^{l}_t (\tilde{x}^{l-1}_t)^\top \,.
\label{eq:eligibility_trace_RMHL}
\end{align}
Similar to NRL's noise-based learning rule defined in Equation~\eqref{eq:final_update}, the synaptic weights for RMHL are then updated by modulating this eligibility trace with the reward prediction error. 

In Table \ref{tab:learning-rules}, we present the complete three learning rules (NRL, BPTT, and RMHL) in a single table for visualization and quick comparison. All updates are written in a per-layer form using $W^l$ and assume a delayed reward received at time $\tau_m$.

\begin{table}[h]
\centering
\caption{Learning rules. The terms in the BPTT update are rearranged for easier comparison. Updates on NRL and RMHL rely on local synaptic information (inputs $\tilde{x}$, noise $\xi$) modulated by global broadcasted scalars (RPE $\delta$, and for NRL, noise impact $\rho$).  In contrast, BPTT relies on non-local gradient information ($\nabla$) propagated from later layers. The updates are given by $W^l_{\tau_m}\gets W^l_{\tau_{m-1}}+\eta \; \Delta W^l_{\tau_m}$.}
\label{tab:learning-rules}
\begin{small}
\begin{tabular}{m{50pt}l}
\vspace{-3pt}
& \\
\textbf{Rule} & \textbf{Update} \\
 \hline
& 
\vspace{-7pt}
\\
NRL &
$\displaystyle
\Delta W^l_{\tau_m} = \delta_{\tau_m}\!\sum_{t=\tau_{m-1}}^{\tau_m}\bar{\xi}^{\,l}_t\,\rho_t\,(\tilde{x}^{l-1}_t)^\top
$ \\
& 
\vspace{-7pt}
\\
 \hline
& 
\vspace{-7pt}
\\
BPTT &
$\displaystyle
\Delta W^l_{\tau_m} = \delta_{\tau_m} \nabla_{W^l_{\tau_{m-1}}}\log\pi_{\theta_{\tau_{m-1}}}(a_t\mid s_t)
$
\\
& 
\vspace{-7pt}
\\
& 
\vspace{-7pt}
\\
 \hline
& 
\vspace{-7pt}
\\
RMHL &
$\displaystyle
\Delta W^l_{\tau_m} = \delta_{\tau_m}\!\sum_{t=\tau_{m-1}}^{\tau_m}\xi^{\,l}_t\,(\tilde{x}^{l-1}_t)^\top
$ \\
 \hline
\end{tabular}
\end{small}
\end{table}

In environments with high reward variability, like Cartpole and Acrobot, we stabilize synaptic updates by scaling the RPE by dividing by $r_{\tau_m}$. This normalization accounts for the non-stationary nature of the return in some tasks, where the reward magnitude scales with trial duration. By treating the error as relative to the current return, we prevent update instability as the agent's performance improves. Hyperparameter values and training details are provided in Appendix \ref{appendix:hyperparams}.  All experiments are implemented in Python using PyTorch~\cite{paszke2019pytorch}. Our models and scripts are available for reproducibility at \url{https://github.com/jesusgf96/noise-based-reward-modulated-learning}.

\section{Results}

\subsection{Performance on control tasks with instantaneous and delayed rewards}

\begin{figure*}[!ht]
    \centering
    \includegraphics[width=\linewidth]{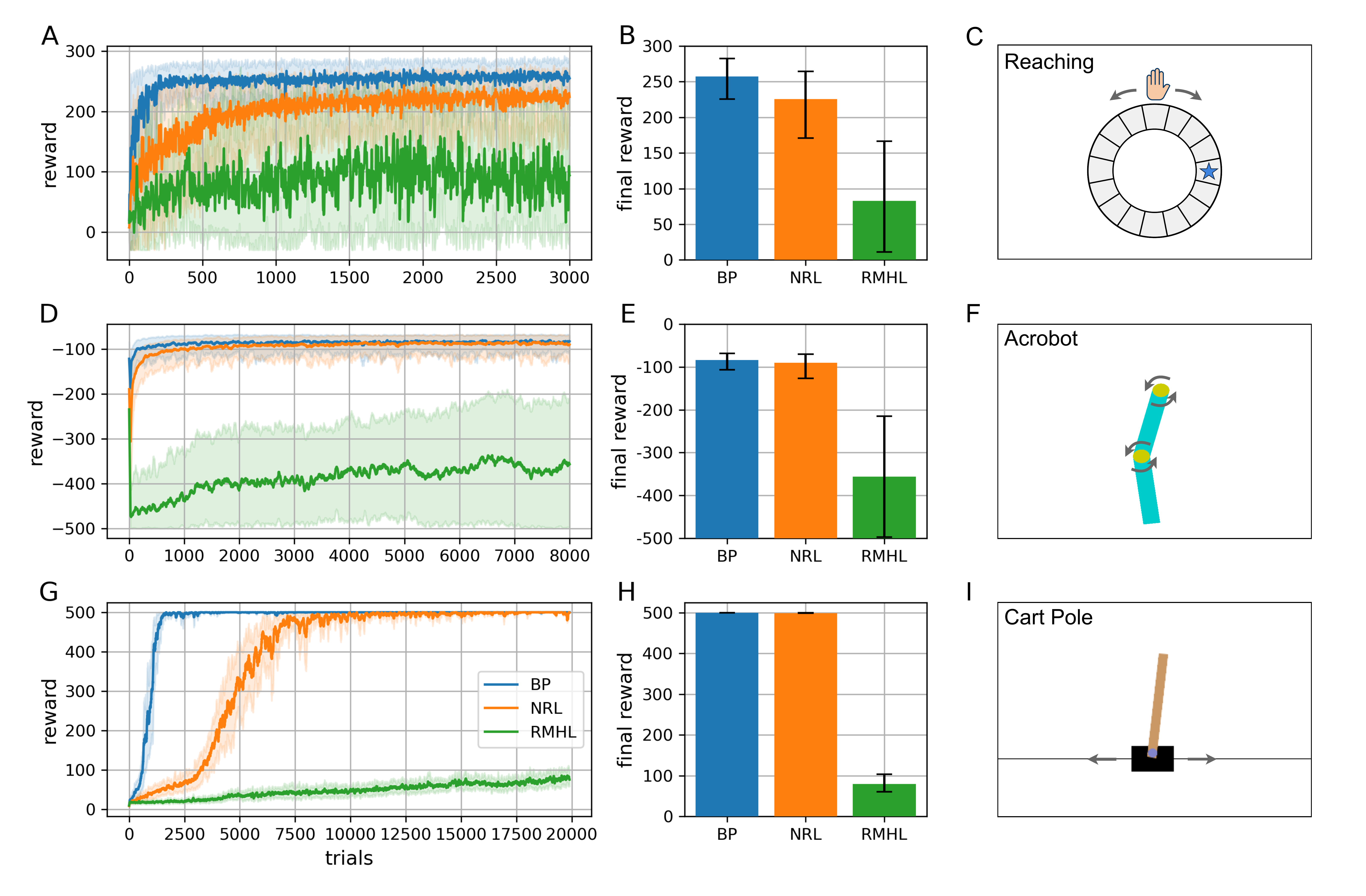}    
    \caption{\textbf{Performance on benchmarks.} A, B, C: Reaching problem. D, E, F: Acrobot problem. G, H, I: Cartpole problem. Left panels: Performance across trials averaged over 5 runs. Centre panels: Final performance (mean of the last 50 trials), averaged over 5 runs. Right panels: Problem visualization.} 
    \label{fig:all_experiments}
\end{figure*}

To ensure comparability with RMHL methods, which typically utilize single-hidden-layer networks, we first conduct experiments with a one-hidden-layer neural network for BP, NRL, and RMHL. Training details are provided in Appendix \ref{appendix:hyperparams}. In Section ``\nameref{scalability}'', we extend the comparison to deeper networks with multiple hidden layers.

First, the Reaching problem~\cite{georgopoulos1986neuronal}, visualized in Fig.~\ref{fig:all_experiments}C, is an immediate reward problem that requires the system to reach and maintain a position at a target on a 1D ring by moving left, right, or remaining stationary. At each step, the system receives information about both the target's position and its own, and a reward is provided based on proximity to the target over fixed-duration trials. Average performance across-trial and final performance (mean of the last 50 trials) are shown in Fig.~\ref{fig:all_experiments}A and Fig.~\ref{fig:all_experiments}B, respectively.

Second, the Acrobot problem~\cite{sutton1995generalization}, visualized in Fig.~\ref{fig:all_experiments}F, involves delayed reward and requires controlling a two-link robotic arm to reach a target height. At each time step, the system receives information about the angles and angular velocities of the two links and chooses one action: clockwise torque, counterclockwise torque, or no torque. Rewards are given based on the speed of completion, with a maximum time allowed. Average performance across trials and final performance (mean of the last 50 trials) are shown in Fig.~\ref{fig:all_experiments}D and Fig.~\ref{fig:all_experiments}E, respectively.

The third and most challenging problem, the Cartpole problem~\cite{barto1983neuronlike}, visualized in Fig.~\ref{fig:all_experiments}I, is a delayed reward problem where the system must balance a pole on a cart by moving left or right. At each time step, the system receives information about the cart's position and velocity, along with the pole's angle and angular velocity, and responds accordingly. Performance is measured by the time the pole remains balanced, with a maximum time allowed. Average performance across trials and final performance (mean of the last 50 trials) are shown in  Fig.~\ref{fig:all_experiments}G and Fig.~\ref{fig:all_experiments}H, respectively.
For all three tasks, NRL achieves a final performance comparable to the gradient-based baseline, BP, demonstrating a vast difference compared to RMHL.

\subsection{Scalability to deeper architectures}
\label{scalability}

Here, we demonstrate that NRL can effectively assign credit in neural networks with multiple hidden layers; a challenging scenario where most biologically plausible algorithms struggle. We use both the Acrobot and the Cartpole tasks for this purpose, as they present a delayed reward problem, making them a more realistic test for credit assignment in reinforcement learning. In this comparison, we include the same baselines, BP and RMHL. Figure~\ref{fig:scalability} displays the results for the Acrobot task and Fig.~\ref{fig:scalability_cartpole} displays the results for the Cartpole tasks for neural networks consisting of two and three hidden layers. The left panels of these two figures show performance across trials, averaged over 5 runs. In contrast, the right panels display the final performance (mean of the last 50 trials), also averaged over 5 runs. 

\begin{figure*}[!ht]
    \centering
    \includegraphics[width=\linewidth]{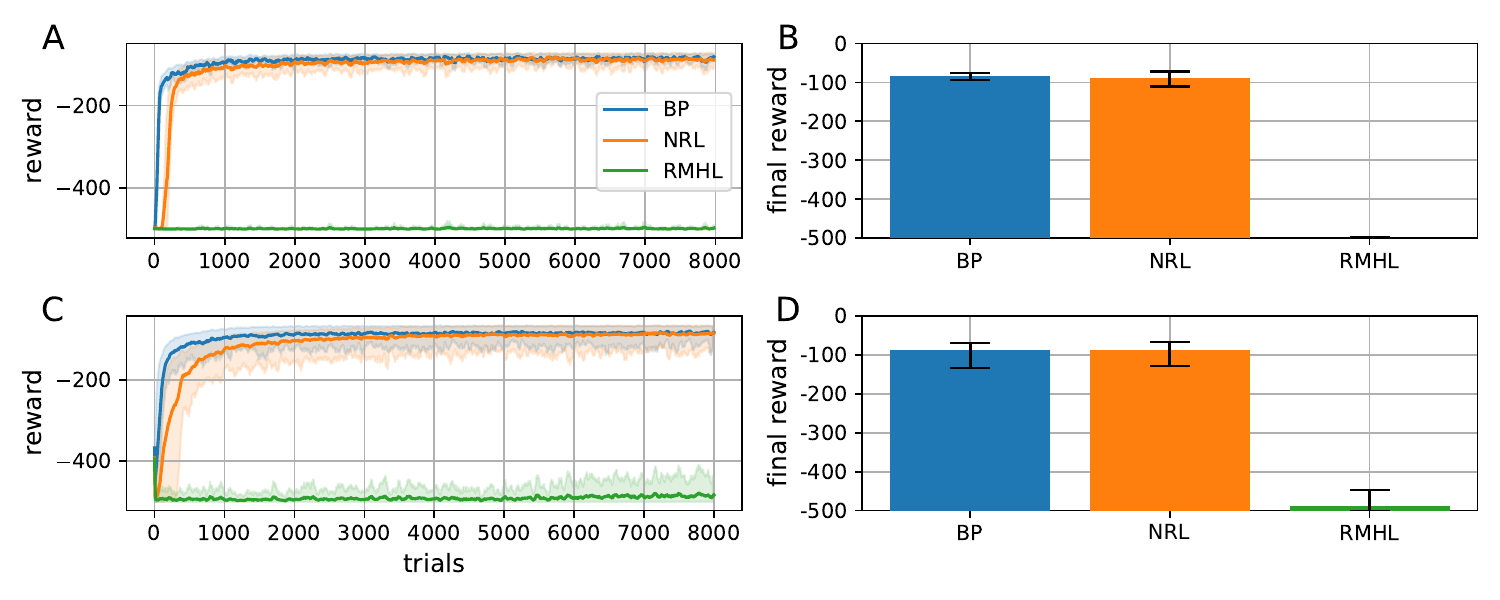}
    \caption{\textbf{Performance on deeper networks for the Acrobot.} A, B: 2-hidden layer networks. C, D: 3-hidden layer networks. Left panels: Performance across trials averaged over 5 runs. Right panels: Final performance (mean of the last 50 trials), averaged over 5 runs.} 
    \label{fig:scalability}
\end{figure*}
\begin{figure*}[!ht]
    \centering
    \includegraphics[width=\linewidth]{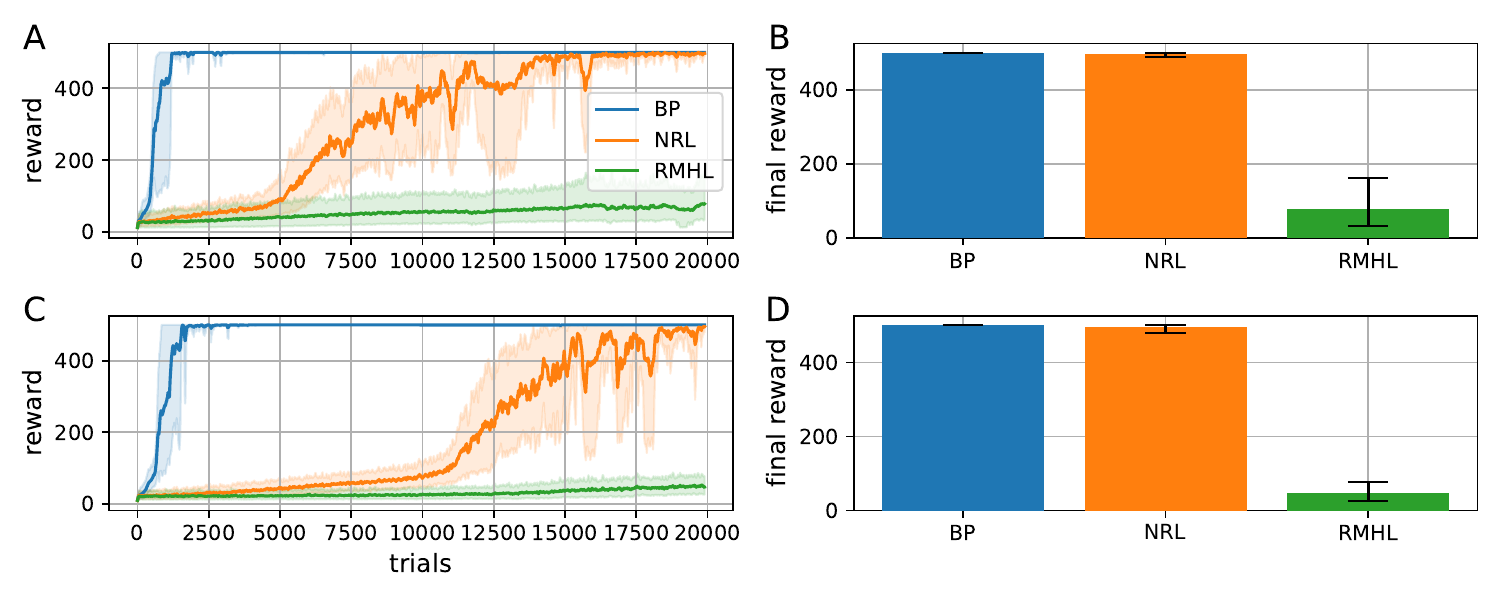}
    \caption{\textbf{Performance on deeper networks for the Cartpole.} A, B: 2-hidden layer networks. C, D: 3-hidden layer networks. Left panels: Performance across trials averaged over 5 runs. Right panels: Final performance (mean of the last 50 trials), averaged over 5 runs.} 
    \label{fig:scalability_cartpole}
\end{figure*}

Our results indicate that NRL successfully learns to solve the tasks, achieving performance comparable to BP. In contrast, RMHL struggles with credit assignment in deeper networks. However, NRL requires more trials to converge as the network depth increases, which is an expected outcome due to the stochastic nature of the updates~\cite{hiratani2022stability}. A similar trend is observed with BP, though to a lesser extent, as its gradient-based updates inherently provide more directed adjustments.

\subsection{Learning using only noisy passes}
\label{approx_clean_pass}

In NRL, $\rho_t$ is calculated as the difference in the network's output between the clean and noisy passes. However, generating a perfectly noiseless pass may be hardware-infeasible or energy-prohibitive in neuromorphic circuits. Instead, we show that the clean network's output can be approximated by averaging the outputs from multiple noisy passes, as the injected noise averages to zero $\lim_{N \to \infty} \frac{1}{N} \sum_{i=1}^{N} \xi_{t,i}^l = 0$.
This assumes that the network dynamics are faster than the environment dynamics, allowing the network to perform multiple forward passes before the environment changes.

To evaluate the accuracy of this approximation, we employ a single hidden-layer network and the Acrobot problem. We chose this task as it offers a representative delayed-reward challenge, unlike the Reaching task. Its stable dynamics also allowed us to clearly separate the approximation error from the variance inherent to the task itself, which was difficult with Cartpole's sensitive initial conditions. The difference between the clean pass output and the averaged noisy passes output was calculated for each timestep and averaged over 500 timesteps, as shown in Fig.~\ref{fig:approximation_error}A.  Furthermore, Fig.~\ref{fig:approximation_error}B illustrates the performance on the Acrobot problem using only noisy forward passes, starting with a minimum of 2 noisy passes. We also extended this evaluation to 10 noisy passes to explore the impact of increasing the number of passes, showing slightly faster initial convergence.

\begin{figure*}[!ht]
    \centering
    \includegraphics[width=\linewidth]{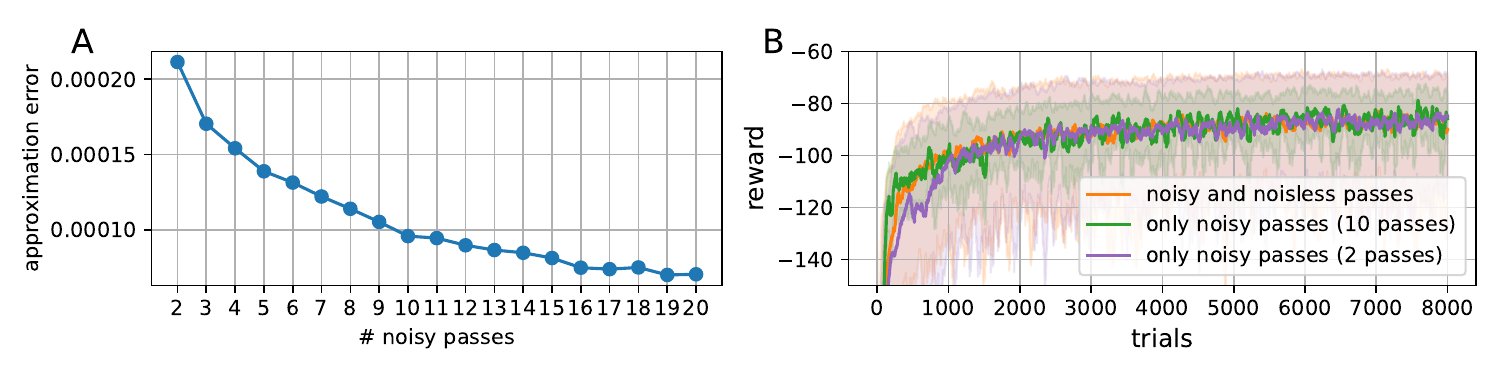}
    \caption{\textbf{Learning using only noisy passes for the Acrobot.} A: Clean pass approximation error. Each data point is computed with an absolute error, averaged over 500 timesteps using the Acrobot problem. B: Acrobot problem. Performance across trials, averaged over 5 runs, with mean, minimum, and maximum values displayed.}
    \label{fig:approximation_error}
\end{figure*}

\section{Discussion}

In this work, we propose NRL, noise-based reward-modulated learning, a novel synaptic plasticity rule that bridges the principles of reinforcement learning and optimization with biologically inspired Hebbian updates. By leveraging stochastic neural activity, our method produces local synaptic updates modulated by a global reinforcement signal. It offers an efficient approach for neural adaptation, particularly suited to neuromorphic systems.  

We employ a top-down methodology, deriving our learning rule from the mathematical principles of gradient-based optimization. We approximate the gradients through directional derivatives and incorporate bottom-up constraints inspired by biological systems. Specifically, the gradient-approximation process is implemented through stochastic neural activity, directly aligning with the inherent noise in neuromorphic substrates~\cite{schuman2022opportunities, roy2019towards, kudithipudi2025neuromorphic} and the brain~\cite{faisal2008noise, stein2005neuronal, mcdonnell2011benefits}. Reinforcement learning mechanisms are realized through reward prediction errors (emulating dopaminergic reward signals~\cite{schultz1997neural, schultz2016dopamine}), eligibility traces, and local synaptic plasticity. These elements also enable learning under uncertainty~\cite{fiser2010statistically}. This approach results in a learning rule that is both theoretically grounded and potentially realizable in event-driven neuromorphic hardware. 

Our results demonstrate the learning effectiveness of NRL, solving both immediate and delayed feedback/reward problems in simulated environments, achieving performance comparable to baselines optimized employing backpropagation. In addition, it shows a significant improvement over RMHL~\cite{legenstein2010reward, miconi2017biologically}, a modern approach that also uses noise as the core of its synaptic plasticity. Our scalability results further demonstrate that NRL effectively assigns credit even in multilayer networks, a challenge for many biologically plausible or local algorithms, indicating its potential for scaling. This is a critical difference from RMHL, which struggles to assign credit in multilayer networks, as evidenced by its minimal performance improvements across trials. Nevertheless, as the depth of the network increases, NRL shows slower convergence due to its noise-driven updates. Beyond network depth, NRL performance is also influenced by the duration of the reward delay. As this delay increases, the eligibility trace accumulates more noise, lowering the signal-to-noise ratio in the gradient estimate. Consequently, although our RPE normalization prevents instability caused by large reward values in long episodes, we expect convergence to slow down as the delay extends, reflecting the difficulty of assigning credit over long time windows. This challenge can be compared with the vanishing gradient problem in BPTT, characterized by a signal decay. Yet, it contrasts with the additive variance accumulation problem seen in NRL, where the learning signal is not lost, but rather obscured by the noise of non-causal perturbations.

Regarding the scope of applicability, our framework suggests specific behaviors in complex scenarios. For high-dimensional observation spaces, the main challenge for perturbation-based learning is the potential variance increase in the gradient estimate. In this scenario, NRL may require more samples for convergence. However, the inclusion of a baseline via the RPE serves as a variance reduction mechanism \cite{williams1992simple}, stabilizing the learning signal even in larger input spaces. When it comes to handling continuous action spaces, NRL, as a policy gradient approximation, is naturally compatible with continuous control. This extension would not require any changes in the learning rule, but simply replacing the categorical output with a Gaussian policy head, where noise perturbs the action mean directly.

NRL shares similarities with node perturbation methods~\cite{flower1992summed, hiratani2022stability, williams1992simple, werfel2003learning, fiete2006gradient}, which also use noise to guide learning. However, they differ in their optimization objective and temporal scope. Standard node perturbation is formulated as a zeroth-order optimization technique for estimating the gradient of an immediate loss function. In contrast, NRL is derived as an approximation of policy gradients, designed to maximize long-term expected return. This distinction is realized through the use of an eligibility trace~\cite{gerstner2018eligibility, izhikevich2007solving}, which keeps a temporary memory of local information (perturbation, activations, etc) until a delayed reward arrives. This effectively assigns credit to past neural states when actions and outcomes are separated in time. In addition, NRL also incorporates variance reduction via the RPE ($\delta_t = r_t - \bar{r}_t$), where the expected reward $\bar{r}_t$ acts as a baseline to stabilize the gradient estimate.

Gradient approximation, a cornerstone of our learning rule's derivation, is also crucial in mathematical optimization and machine learning, enabling learning when exact gradients are computationally expensive or unavailable. Techniques such as forward gradient methods~\cite{baydin2022gradients} and zeroth-order optimization~\cite{liu2020primer, chen2023deepzero} have emerged, utilizing stochastic perturbations and function value comparisons to estimate gradients. Recent advances, such as employing local auxiliary networks for informed gradient guesses~\cite{fournier2023can}, have significantly improved the alignment of these approximations with exact gradients, reducing the performance gap with backpropagation. These variance-reduction strategies could similarly enhance our approach by improving the quality of the updates and scalability to more complex architectures. Furthermore, compared to traditional gradient-based methods like backpropagation, gradient approximation methods offer a key advantage in computationally constrained scenarios, or when using non-differentiable networks, such as spiking neural networks~\cite{yamazaki2022spiking, tavanaei2019deep}.

Additionally, the local-update structure of NRL provides direct compatibility with neuromorphic substrates \cite{schuman2022opportunities}, where only locally available information in the synapse is prioritized and global communication is severely penalized. This eliminates the need for the high-bandwidth global data movement and large memory buffers required by backpropagation, the primary source of energy consumption in deep learning models, addressing the fundamental energy and latency bottlenecks of conventional Von Neumann systems. Furthermore, NRL's reliance on noise for gradient approximation offers a paradigm shift: it transforms the device stochasticity common in neuromorphic circuits from an engineering challenge into a computational resource. The inherent noise in NRL is also advantageous for reinforcement learning setups, as it naturally introduces the stochasticity required for exploration \cite{kaelbling1996reinforcement, ladosz2022exploration}. Injected noise perturbs neural states, promoting deviations from expected values and thus introducing variations in the agent’s policy. Recent studies on noise-based adaptation \cite{Fernandez2024Ornstein} have demonstrated that noise-driven learning mechanisms can effectively balance exploration and exploitation, particularly in volatile environments.

Despite its promise, our work has limitations that suggest future research directions. The current requirement for both noisy and clean forward passes (or averaging multiple noisy passes) per synaptic update, while theoretically simple, introduces a time-latency and computational overhead. In practice, this overhead is mitigated by the accelerated dynamics of neuromorphic hardware, which typically operates at microsecond timescales, orders of magnitude faster than real-world environments \cite{davies2018loihi}. Thus, performing multiple forward passes potentially has a negligible impact on latency. From an energy perspective, the cost of extra passes is offset by eliminating the high-bandwidth data transport and memory buffering required by gradient-based methods, such as BP \cite{Fernandez2024}, which remain the primary energy bottlenecks in conventional learning systems. However, to further enhance this method, future work could explore alternative ways to estimate noise impacts, eliminating the need for multiple passes. 

In addition, while we focused on a rate-based model, the deployment of NRL on physical hardware like Intel Loihi~\cite{davies2018loihi} requires its extension to spiking neural networks (SNNs). Adapting NRL to this domain presents specific implementation challenges. For instance, the continuous eligibility traces used here would need to be replaced by synaptic tracers that accumulate discrete spike events via exponential decay kernels \cite{izhikevich2007solving}. Similarly, the injected noise used for gradient approximation could translate into stochastic membrane potential fluctuations or probabilistic firing thresholds \cite{maass2014noise}, transforming perturbations into spike-timing variability. Also, while our derivation assumes Gaussian noise, physical hardware often exhibits non-Gaussian stochasticity. However, as elaborated in Appendix \ref{appendix:hardware_noise}, the central limit theorem suggests that the aggregation of these discrete noise sources at the neuronal level effectively approaches a Gaussian distribution. Thus, our theoretical framework is expected to remain robust in physical neuromorphic systems.

Furthermore, the simplicity of our RPE calculation limits its robustness in more complex tasks requiring long-term planning. Future development should incorporate more sophisticated temporal difference learning mechanisms~\cite{sutton1988learning, sutton2018reinforcement}, such as a learned value function within an actor-critic architecture~\cite{konda1999actor}, to provide a richer temporally sensitive estimate of rewards. Finally, while backpropagation may offer faster convergence in unconstrained settings, recent studies~\cite{dalm2023effective, Fernandez2024}, including ours, indicate that this gap can be narrowed for specific networks and tasks.

Finally, regarding sample efficiency, NRL is an on-policy algorithm theoretically grounded in the policy gradient theorem \cite{sutton1999policy}. Although off-policy methods can improve data efficiency, they typically rely on mechanisms that conflict with biological and neuromorphic constraints. Incorporating off-policy mechanisms, such as experience replay \cite{mnih2015human} or importance sampling \cite{precup2000eligibility}, would require centralized memory buffers and global data transport, violating the principle of synaptic locality. Crucially, this restriction allows us to isolate and evaluate the performance of the learning rule itself, avoiding the obscuring factors introduced by add-ons. Future work could investigate how such off-policy augmentations can be adapted to the neuromorphic context to further enhance sample efficiency.

Our findings ultimately highlight the potential for unifying optimization methods and neurobiological principles. By drawing inspiration from both domains, we have developed a learning rule that not only offers an alternative to backpropagation but is also algorithmically suited for low-energy, event-driven neuromorphic hardware. In the same way, the constraints of biological realism can guide the development of more efficient and robust artificial learning algorithms. This synergistic approach provides a crucial new path for creating low-energy, fast, and adaptive AI systems \cite{zador2023catalyzing, marblestone2016toward}.

\section*{Acknowledgments}
This publication is part of the DBI2 project (024.005.022, Gravitation), which is financed by the Dutch Ministry of Education (OCW) via the Dutch Research Council (NWO).

\bibliographystyle{plain}
\bibliography{references}

\newpage
\appendix

\section{Gradient approximation using directional derivatives} 
\label{appendix:directional_derivatives}

Consider a function $g(\theta)$, where $\theta = (W_1,\ldots, W_L)$ are its parameters. The gradient of this function is defined as follows:
\begin{definition}
The gradient $\nabla g $ is a vector indicating the direction of the steepest ascent of the function $g$, with components as partial derivatives of $g(\theta)$:
\begin{equation}
\nabla g  
= \left[\frac{\partial g}{\partial \theta}\right]^\top 
= \left[ \frac{\partial g}{\partial W_1}, \ldots, \frac{\partial g}{\partial W_L} \right]^\top \,.
\end{equation}
\end{definition}

While $\nabla g$ captures the rate of change of 
$g$ in the steepest direction, a directional derivative gives the rate of change in a specified direction. For a unit vector $v = \flatfrac{\epsilon}{||\epsilon||}$, normalized via the Euclidean norm, we define the directional derivative as: 
\begin{definition}
The directional derivative of $g(\theta)$
along a unit vector $v =(v_{1},\ldots ,v_{n})$
is defined by the limit
\begin{equation}
\nabla _{v}{g}(\theta)=\lim _{h\to 0}{\frac {g(\theta +h v )-g(\theta)}{h}} \,,
\label{eq:dirder}
\end{equation}
where $h$ is a small step size.
\end{definition}
For a sufficiently small $h$ we numerically measure the directional derivative as 
\begin{equation}
\nabla _{{\epsilon}}{g}(\theta ) = \frac {g(\theta + h \epsilon )-g(\theta)}{h||\epsilon||} \,.
\end{equation}
This directional derivative can also be measured as a projection of
$\nabla g$ in the direction $v$, following the relation:
\begin{equation}
    \nabla _{v}{g}(\theta) = v \cdot \nabla{g}(\theta).
\end{equation}
We can now formally demonstrate how gradients can be approximated using directional derivatives.
\begin{theorem}
Let $\epsilon \sim \mathcal{N}(0, \sigma^2 I_n)$ (probability distribution $p(\epsilon)$) where $n$ is the number of dimensions in $\theta$. Exact gradients can be written in terms of directional derivatives using expectations
\begin{equation}
\nabla g(\theta) 
= n \mathbb{E}_{p(\epsilon)} \left[ v^\top \nabla_{v} g(\theta) \right] \,.
\label{eq:grad_direct_der}
\end{equation}
\label{thm:exact_grads_equiv}
\end{theorem}
\begin{proof}
\begin{align*}
\nabla g(\theta) & = n \mathbb{E}_{p(\epsilon)} \left[ v^\top \nabla_{v} g(\theta) \right] \,
\\ & = n \mathbb{E}_{p(\epsilon)} \left[ v v^\top \nabla g(\theta ) \right] \,
\\ & = n \mathbb{E}_{p(\epsilon)} \left[ v v^\top\right] \nabla g(\theta )  \,
\\& = n \mathbb{E}_{p(\epsilon)} \left[ \frac{\epsilon}{||\epsilon||}\frac{\epsilon^\top}{||\epsilon||} \right]  \nabla g(\theta ) \,
\\& = n \frac{1}{n} I_n  \nabla g(\theta ) = \nabla g(\theta) \,
\end{align*}
as we assume $p(\epsilon) = \mathcal{N}(0,\sigma^2 I_n)$, which gives $\mathbb{E}[v v^\top] = \mathbb{E}\left[\frac{\epsilon}{||\epsilon||}\frac{\epsilon^\top}{||\epsilon||} \right] = \frac{1}{n} I_n$. 
\end{proof}

Now, consider gradient descent in the direction of the gradient $\nabla g(\theta)$ as
\begin{equation}
    \theta_{t+1} = \theta_t + \alpha \nabla g(\theta_t).
\end{equation}
This update can be reformulated using directional derivatives.
\begin{theorem}
Let $\eta = \frac{\alpha n}{hK}$ and $\epsilon^{(i)} \sim \mathcal{N}(0, \sigma^2 I_n)$. 
Gradient descent is equivalent to the update rule
\begin{equation}
\theta_{t+1} = \theta_t + \alpha \sum_{i=1}^K \left[ \frac{\epsilon^{(i)}}{h||\epsilon^{(i)}||^2} \left[ g\left(\theta + h\epsilon^{(i)} \right)-g\left(\theta \right) \right]^\top \right] 
\end{equation}
in the limit when $h \to 0$ and $K\rightarrow \infty$.
\label{thm:grad_desc_equiv}
\end{theorem}
\begin{proof}
    \begin{align*}
        \nabla g(\theta ) &= n \mathbb{E}_{p(\epsilon)} \left[ v \nabla_{v} g(\theta )^\top \right] \,
        \\ &= \lim_{h \rightarrow 0}  n \mathbb{E}_{p(\epsilon)} \left[ v \left[ \frac {g(\theta + h \epsilon )-g(\theta )}{h||\epsilon||} \,\right]^\top \right] \,
        \\ &= \lim_{h \rightarrow 0}  n \mathbb{E}_{p(\epsilon)} \left[ \frac{\epsilon}{||\epsilon||} \left[ \frac {g(\theta + h \epsilon )-g(\theta )}{h||\epsilon||} \,\right]^\top \right]\,
        \\ &= \lim_{h \rightarrow 0}  n \mathbb{E}_{p(\epsilon)} \left[ \frac{\epsilon}{h||\epsilon||^2} \left[ g(\theta + h \epsilon )-g(\theta ) \,\right]^\top \right].
    \end{align*}
    Now, substituting the expectation with a sampling under some empirical distribution:
    \begin{equation*}
        \nabla g(\theta ) 
        = \lim_{h \rightarrow 0} \lim_{K \rightarrow \inf} \frac{n}{hK} \sum_{i=1}^K \left[ \frac{\epsilon^{(i)}}{||\epsilon^{(i)}||^2} \left[ g\left(\theta + \epsilon^{(i)} \right)-g\left(\theta \right) \right]^\top \right]. 
    \end{equation*}
Finally, defining gradient descent on parameters $\theta$ as $\theta_{t+1} = \theta_t + \alpha \nabla g(\theta)$:
\begin{align*}
    \theta_{t+1} &= \theta_t + \alpha \nabla g(\theta)
    \\ &= \alpha \lim_{h \rightarrow 0} \lim_{K \rightarrow \inf} \frac{n}{hK} \sum_{i=1}^K \left[ \frac{\epsilon^{(i)}}{||\epsilon^{(i)}||^2} \left[ g\left(\theta + \epsilon^{(i)} \right)-g\left(\theta \right) \right]^\top \right] 
    \\ &= \lim_{h \rightarrow 0} \lim_{K \rightarrow \inf} \frac{\alpha n}{hK} \sum_{i=1}^K \left[ \frac{\epsilon^{(i)}}{||\epsilon^{(i)}||^2} \left[ g\left(\theta + \epsilon^{(i)} \right)-g\left(\theta \right) \right]^\top \right]. 
\end{align*}

\end{proof}
In practice, small perturbations, controlled by $h$, and limited numbers of noise samples are sufficient to approximate the gradient. For $K=1$, analogous to single-sample updates in stochastic gradient descent:
\begin{equation}
\theta_{t+1} = \theta_t + \eta \frac{\epsilon}{||\epsilon||^2} \rho  
\end{equation}
with $\rho = \left[ g(\theta + h\epsilon )-g(\theta) \right]^\top$.

\section{Model hyperparameters and training details} 
\label{appendix:hyperparams}

All hyperparameters were carefully tuned per method and problem to ensure fair comparisons across methods. The number of episodes was chosen to illustrate the convergence of our method relative to baseline methods, with 8000, 20000, and 1000 episodes used for the Acrobot, Cartpole, and Reaching problems, respectively. 

In all our experiments, we set the smoothing factor for reward estimation to $\lambda = 0.66$. This value was empirically found to offer a robust trade-off across all tasks, striking a balance between recent values (for quick adaptability) and a longer history (for robustness against rapid reward fluctuations). We observed that this choice was sufficiently efficient to be applied uniformly across all environments without the need for task-specific adjustment.

Each neural network consisted of an input, hidden, and output layer. Input units corresponded to environment observation elements, and output units to possible actions. Specifically, Acrobot, Cartpole, and Reaching used 6, 4, and 32 input units and 3, 2, and 2 output units, respectively, with hidden layer sizes of 64, 64, and 128. 
Table \ref{table:hyperparams} summarizes the learning rate $\eta$ and noise standard deviation $\sigma$ for each method and problem. Higher values of these parameters could lead to unstable training, while lower values may result in slower learning.

\begin{table}[!ht] 
\caption{Learning rate $\eta$ and noise standard deviation $\sigma$ for the different learning algorithms across problems. Dashes indicate that the parameter is not used.}
\centering
\begin{small}
\begin{tabular}{l|c c c c} 
  & & \textbf{BP} & \textbf{Ours} & \textbf{RMHL}\\
 \hline
 \multirow{2}{*}{\textbf{Acrobot}} & $\eta$ & 5e-3 & 5e-2
 & 5e-2 \\ 
 \cline{2-5} 
   & $\sigma$ & - & 1e-3 & 1e-3 \\
 \hline
  \multirow{2}{*}{\textbf{Cartpole}}  & $\eta$ & 5e-3 & 5e-2 & 1e-2 \\ 
 \cline{2-5} 
  & $\sigma$ & - & 1e-3 & 1e-1 \\
 \hline
  \multirow{2}{*}{\textbf{Reaching}}  & $\eta$ & 1e-2 & 1e-2 & 1e-1 \\ 
 \cline{2-5} 
  & $\sigma$ & - & 1e-3 & 1e-1 \\
\hline
\end{tabular}
\end{small}
\label{table:hyperparams}
\end{table}

\section{Noise distributions in neuromorphic hardware} 
\label{appendix:hardware_noise}

In our derivation of the NRL update rule, we assume that the perturbative noise $\xi$ follows a Gaussian distribution. This assumption is crucial for maintaining equality between the expected update and the true gradient. However, in physical neuromorphic implementations, noise sources at the device level may follow different distributions. 

In practice, neuromorphic substrates manifest stochasticity through physical processes. Thermal or Johnson–Nyquist noise \cite{johnson1928thermal}, arising from charge carriers, is inherently Gaussian and is always present in analog circuits. However, other sources are non-Gaussian. Shot noise \cite{schottky1918spontane}, caused by the discrete nature of electric charge crossing a barrier, follows a Poisson distribution. Random telegraph noise \cite{kirton1989noise}, common in memristive devices, manifests as discrete switching between two states, following a bimodal distribution.

Although individual synaptic or neuronal noise sources may follow non-Gaussian distributions, the NRL update rule operates on the aggregated activity influencing the neuron's state. A fundamental property of neural integration is the summation of inputs from multiple presynaptic sources. According to the central limit theorem (CLT) \cite{fischer2011history}, the sum of a large number of independent random variables, regardless of their original distribution, provided they have finite variance, converges to a Gaussian distribution. Consequently, the Gaussian assumption in $\xi$ used in our derivation is not only a mathematical convenience, but a statistically valid approximation of the effective aggregate noise seen in hardware implementations or biological networks.

Even in scenarios with low connectivity where the CLT approximation is not perfect, the learning rule is expected to remain robust. While Theorem \ref{thm:exact_grads_equiv} relies on this Gaussian assumption for exact equality between the expected update and the true gradient, in practice, optimization does not require the update vector to be identical to the true gradient, but to have a positive projection onto it, i.e., a positive cosine similarity. This principle is well-established in the literature, where algorithms, such as feedback alignment \cite{lillicrap2016random}, direct feedback alignment \cite{nokland2016direct}, or sign-symmetry \cite{liao2016important}, demonstrate that approximate gradients are sufficient for convergence, as long as they remain within $90^{\circ}$ of the steepest descent direction. Therefore, skewed non-Gaussian noise may reduce convergence speed by lowering this cosine similarity, but it does not orthogonalize the update vector, thus maintaining a valid descent trajectory.

\end{document}